\documentclass[accepted]{uai2026} % after acceptance, for a revised version; 
\PassOptionsToPackage{hyphens}{url} % allow breaks at hyphens too
\usepackage{xurl}                  % much better \url line breaks than url.sty
\usepackage[american]{babel}
\usepackage{natbib} % has a nice set of citation styles and commands
\usepackage{mathtools} % amsmath with fixes and additions
\usepackage{booktabs} % commands to create good-looking tables
\usepackage{tikz} % nice language for creating drawings and diagrams
\usepackage{amsmath,amsthm,amssymb}
\DeclareMathOperator{\supp}{supp}
\theoremstyle{plain}
  \newtheorem{theorem}{Theorem}
  
  \newtheorem{proposition}{Proposition}
  \newtheorem{lemma}{Lemma}
  \newtheorem{corollary}{Corollary}

\theoremstyle{definition}
  \newtheorem{definition}{Definition}

\theoremstyle{definition}
  
  \newtheorem{remark}{Remark}

\usepackage{graphicx}
\usepackage{tikz}
\usetikzlibrary{arrows.meta,positioning,matrix}

\title{What Capable Agents Must Know: Selection Theorems for Robust Decision-Making under Uncertainty}

\author[1]{\href{mailto:<anayebi@cs.cmu.edu>?Subject=Your UAI 2026 paper}{Aran Nayebi}{}}
\affil[1]{%
    Machine Learning Department and Neuroscience \& Robotics Institutes\\
    Carnegie Mellon University\\
    Pittsburgh, Pennsylvania, USA
}
  
\begin{document}
\maketitle

\begin{abstract}
As artificial agents become increasingly capable, what internal structure is \emph{necessary} for an agent to act competently under uncertainty? 
Classical results show that optimal control can be \emph{implemented} using belief states or world models, but not that such representations are required.
We prove quantitative ``selection theorems'' showing that strong task performance (low \emph{average-case regret}) forces world models, belief-like memory and---under task mixtures---persistent regime-tracking variables resembling functional primitives of emotion, along with informational modularity under block-structured tasks.
Our results cover stochastic policies, partial observability, and evaluation under task distributions, without assuming optimality, determinism, or access to an explicit model.
Technically, we reduce predictive modeling to binary ``betting'' decisions and show that regret bounds limit probability mass on suboptimal bets, enforcing the predictive distinctions needed to separate high-margin outcomes. In fully observed settings, this yields approximate recovery of the interventional transition kernel; under partial observability, it implies necessity of predictive state and belief-like memory, addressing an open question in prior world-model recovery work.
\end{abstract}

\section{Introduction}
\label{sec:intro}

What internal structure is \emph{necessary} for an agent to robustly act competently under uncertainty?

Classical results in control and reinforcement learning show that optimal behavior can be implemented using belief states or world models \citep{sondik1971,kaelbling1998}. 
These results are constructive: they show that an optimal controller \emph{can} be expressed as a function of a sufficient statistic. 
They do not establish that predictive internal state is \emph{required}. 
An architecture might be capable of belief-based control without being forced to implement predictive structure by the demands of its task distribution. 
Our aim is to close this gap, in the sense of ``selection-style'' arguments articulated by~\citet{wentworth2021selection}.

Across decision theory, control, and learning theory, broad performance requirements often imply structural constraints. 
Classical representation theorems show that agents satisfying rationality axioms behave \emph{as if} maximizing expected utility~\citep{von1947theory,savage1954foundations}, and later axiomatic work~\citep{karny2012axiomatisation,karny2020axiomatisation} studies the ordering of policies in closed-loop dynamic decision problems via local functionals, showing that such orderings induce probabilistic modeling of uncertainty in the optimized decision process.
%where objectives are encoded as ideal trajectory distributions and optimal strategies minimize KL divergence to those ideals (though specifying such ideals may be nontrivial in practice).
The Good Regulator Theorem asserts that regulation requires modeling the system \citep{conant1970every}, a requirement formalized in linear control by the Internal Model Principle \citep{francis1976internal}.
No-regret guarantees constrain the information needed to avoid systematic loss \citep{blackwell1956analog,foster1997calibrated}. 
However, these approaches either rely on strong axioms, target specialized and exactly optimal regulation settings, or stop short of representation-level necessity conclusions.

\begin{figure}[t]
\centering
\resizebox{\linewidth}{!}{%
\begin{tikzpicture}[
    font=\footnotesize,
    >=Stealth,
    task/.style={
      draw,
      rounded corners=2pt,
      align=center,
      text width=30mm,
      minimum height=8.5mm,
      inner sep=2.5pt
    },
    resultbox/.style={
      draw,
      rounded corners=2pt,
      align=center,
      text width=30mm,
      minimum height=8.5mm,
      inner sep=2.5pt
    },
    head/.style={
      align=center,
      font=\bfseries\small
    },
    title/.style={
      align=center,
      font=\bfseries\small
    },
    arrow/.style={->, thick}
]

% Top annotation
\node[title] at (2.7,0.65)
  {Selection depends on the task family};

% Headers
\node[head] (h1) at (0,0) {Diagnostic task family};
\node[head] (h2) at (5.4,0) {Internal structure selected};

% Left column
\node[task] (t1) at (0,-0.9)
  {Fully observed\\ action-conditioned bets\\ (\S4, Thm.~1)};

\node[task] (t2) at (0,-2.15)
  {Threshold bets on\\ predictive tests\\ (\S5.2, Thms.~2-4)};

\node[task] (t3) at (0,-3.4)
  {Paired-history\\ distinguishing tests\\ (\S5.3, Thm.~5)};

\node[task] (t4) at (0,-4.65)
  {Block-structured\\ test families\\ (Cor.~3)};

\node[task] (t5) at (0,-6.0)
  {Regime-shift / mixture\\ test families\\ (Cor.~4)};

\node[task] (t6) at (0,-7.45)
  {Complete minimal\\ test families\\ (Cor.~5)};

% Right column
\node[resultbox] (o1) at (5.4,-0.9)
  {World model /\\ transition knowledge};

\node[resultbox] (o2) at (5.4,-2.15)
  {Predictive state\\ (PSR coordinates)};

\node[resultbox] (o3) at (5.4,-3.4)
  {Memory separating\\ aliased histories};

\node[resultbox] (o4) at (5.4,-4.65)
  {Modular internal\\ structure};

\node[resultbox] (o5) at (5.4,-6.0)
  {Persistent regime\\ variable\\ (tradeoff-tracking)};

\node[resultbox] (o6) at (5.4,-7.45)
  {Representational\\ convergence\\ (invertible recoding)};

% Arrows
\draw[arrow] (t1.east) -- (o1.west);
\draw[arrow] (t2.east) -- (o2.west);
\draw[arrow] (t3.east) -- (o3.west);
\draw[arrow] (t4.east) -- (o4.west);
\draw[arrow] (t5.east) -- (o5.west);
\draw[arrow] (t6.east) -- (o6.west);

\end{tikzpicture}%
}
\caption{Different diagnostic task families select different internal structure, from world models and predictive state to memory, modularity, persistent regime variables, and representational convergence.}
\label{fig:selection-overview}
\end{figure}

\textbf{Our contribution.}
We prove quantitative selection theorems showing that low \emph{average-case regret} on structured families of action-conditioned prediction tasks forces an agent to implement predictive, structured internal state (visualized in Fig.~\ref{fig:selection-overview}).

Our technical approach reduces predictive modeling to binary ``betting'' goals. 
A regret decomposition shows that average normalized regret bounds directly control the probability mass assigned to suboptimal bets. 
When the evaluation distribution places nontrivial mass on large-margin tests, this forces the agent’s internal memory to refine the predictive partition induced by those tests (Theorems~\ref{thm:fo_avg_stoch}–\ref{thm:memory}). 
In fully observed environments, this yields approximate recovery of the interventional transition kernel (Corollary~\ref{cor:causal_content_expanded}); in partially observed environments, it yields quantitative no-aliasing bounds for belief-like memory, addressing an open question posed by~\citet{richens2025}. 
We also show that~\citet{pearl2009causality} Level 2 interventions are recoverable, but Level~3 counterfactuals are not (Corollary~\ref{cor:no_L3}).

Our results differ from recent world-model recovery work~\citep{richens2024robust,richens2025} in three key respects: \textbf{(i)} we assume only average-case regret rather than worst-case optimality; \textbf{(ii)} our results hold under stochastic policies, which have both had a long history in reinforcement learning~\citep{witten1977adaptive,williams1992simple,sutton1998reinforcement} and are commonly used in modern deep learning algorithms, such as the Dreamer family ~\citep{hafner2019dream,hafner2020mastering,hafner2023mastering,hafner2025training}, PPO~\citep{schulman2017proximal}, along with many others (e.g.~\citep{hansen2023td,wang2024efficientzero} to name a few); and \textbf{(iii)} unlike their work and later recent extensions, we derive necessity results under \emph{partial observability} rather than focusing solely on explicit recovery in fully observed settings~\citep{khetarpal2026affordances,harwood2026information} or under fully observed goals~\citep{cifuentes2026worldmodels}, directly addressing an \emph{open question} raised by~\citet{richens2025}.

\textbf{Structure from task families.}
Beyond predictive modeling and memory, we also show that structured evaluation distributions impose further constraints. 
Block-structured tests select for informational modularity (Corollary~\ref{cor:modularity}); mixtures of regimes select for regime-sensitive internal state (Corollary~\ref{cor:mixtures}); and under minimality assumptions, any two vanishing-regret agents must \emph{representationally converge} on decision-relevant partitions up to invertible recoding (Corollary~\ref{cor:rep_match}). 

Taken together, these results formalize a simple principle:
\begin{quote}
Robust generalization under uncertainty selects for the predictive internal structure tested by the evaluation task family.
\end{quote}
They separate representation \emph{necessity} from representation \emph{recovery} and provide a regret-based route from empirically meaningful competence guarantees on \emph{specified} task families to concrete constraints on internal organization.
After all, no representation theorem can force an agent to distinguish internal states that are never tested by the goals.

\section{Related Work}
Our results are framed in the standard POMDP setting, where posterior belief is a sufficient statistic for optimal control \citep{sondik1971,kaelbling1998}. 
However, these classical results are constructive: they show that optimal behavior \emph{can} be expressed in terms of belief, not that predictive state is \emph{required}.

\citet{bennett2023emergent,bennett2023optimal,bennett2024complexity, bennett2025build,bennett2025formal,bennett2026regret} develops a distinct weakness-maximization framework in which successful adaptation is argued to favor causal-identity constructions separating intervention from observation; however, unlike our regret-based selection theorems, it proceeds via task extensions and additional exchangeability and representation/incentive assumptions, and does not establish direct analogues of our quantitative recovery, partial-observability necessity, or representational convergence results.
Recent philosophical work~\citep{herrmann2026bayesian} has also explored reducing interventional reasoning to probabilistic reasoning over enriched variable spaces, though in a distinct formal setting from the agent-based, regret-theoretic framework considered here.

Our notion of tests follows predictive-state representations (PSRs) \citep{littman2001psr,singh2004psr,boots2011psr}, which represent state via predictions of action-conditioned futures rather than latent variables. 
Unlike the PSR literature, which treats predictive state as sufficient for control, we derive it as \emph{necessary}: low \emph{average-case} regret on action-conditioned prediction tasks forces an agent to compute the predictive distinctions needed to separate high-margin outcomes.
Technically, our core inequality instantiates a standard margin-style regret decomposition \citep{bartlett2006}, but uses it to derive representation-theoretic constraints rather than supervised generalization guarantees.
The betting-goal reduction is related to elicitation, proper scoring rules, and game-theoretic/imprecise probability \citep{savage1971scoring,gneiting2007,dempster2008upper, shafer2005probability}, though we use precise success probabilities rather than truthful reports or lower/upper-probability protocols.

Our work is complementary to \citet{richens2024robust,richens2025}, who show that under strong competence assumptions in fully observed environments one can \emph{recover} a transition model from an agent's policy.
We instead study stochastic policies, partial observability, and \emph{average-case} regret over a distribution of prediction tasks, and derive necessity results rather than recovery procedures. 
In particular, we extend our selection argument to partially observed environments, giving quantitative no-aliasing bounds for belief-like memory---addressing an open question raised by \citet{richens2025}.

\section{Notation and Constants}
\label{sec:notation}
Consider a one-step decision between two actions $L$ and $R$ with success probabilities $u_L,u_R\in[0,1]$.
Let a (possibly stochastic) policy choose $L$ with probability $q\in[0,1]$ and $R$ with probability $1-q$.
Then the achieved success probability is
\begin{equation}\label{eq:success}
\begin{split}
V &= \Pr(\text{success}\mid L)\Pr(L) + \Pr(\text{success}\mid R)\Pr(R)\\ 
&= q\,u_L + (1-q)\,u_R.
\end{split}
\end{equation}
Let the optimal success probability be denoted as $V^\star:=\max_{q \in [0,1]} V$.

Define the normalized regret as
\begin{equation}\label{eq:regret-def}
\delta:=1-\frac{V}{V^\star},
\end{equation}
assuming $V^\star>0$ (this is without loss of generality, as a $V^\star=0$ will imply that the goal is trivially unsatisfiable).

For any $\gamma\in(0,\tfrac12)$, with $\gamma \to 1/2$ only as a limit, define the following constants, which will be used throughout:
\begin{equation}\label{eq:constants}
c(\gamma):=\frac{4\gamma}{1+2\gamma},\qquad t_\gamma:=\sqrt{\frac{1+2\gamma}{1-2\gamma}}\;\ge 1.
\end{equation}

\section{World model recovery in fully observed environments}
Let $E=(\mathcal S,\mathcal A,P,\mu_0)$ be an environment with finite state space $\mathcal S$ and action space $\mathcal A$, with $|\mathcal A| \ge 2$, where $P(s'\mid s,a)$ denotes the one-step transition probabilities and $\mu_0$ the initial-state distribution.
We assume the environment is fully observed (the agent observes $s,s'$ exactly), stationary (transition probabilities do not drift over time), and that actions influence transitions (i.e., there exist $s,a,a',s'$ such that $P(s'\mid s,a)\neq P(s'\mid s,a')$).
Additionally, we assume the environment is communicating, meaning that for any $s,s'\in\mathcal S$ there exists a finite action sequence that reaches $s'$ from $s$ with positive probability, ensuring the agent can in principle carry out the diagnostic goals from any start state, thereby ruling out environments with permanently isolated regions (rather than realistic control problems).

We now can define the goal family. 
Specifically, we define our bets over the agent's successful completion of it:
\label{sec:wm-fo}
\begin{definition}[Composite goal family $G^{(n)}_{s,a,s',k}$]\label{def:goal}
Fix $s,s'\in\mathcal S$, an action to be tested $a\in\mathcal A$, an integer $n\ge 1$, and a threshold $k\in\{0,1,\dots,n\}$.
Pick any two initial \emph{marker} actions $L,R\in\mathcal A$ (used only to select a branch at $t=0$).

For an infinite trajectory $\tau=(S_0,A_0,S_1,A_1,\dots)$ define the attempt times
\begin{equation*}
\begin{split}
&T_1(\tau):=\inf\{t\ge 1:\ S_t=s,\ A_t=a\},\\
&T_{i+1}(\tau):=\inf\{t>T_i(\tau):\ S_t=s,\ A_t=a\},
\end{split}
\end{equation*}
with the convention $\inf\emptyset=\infty$.
Thus, $T_i$ is the $i$-th occurrence of $(S_t = s, A_t =a)$ along $\tau$, if it occurs; otherwise, $T_i = \infty$. 

Define success indicators
\begin{equation*}
\begin{split}
& X_i(\tau):=\mathbf 1\{T_i(\tau)<\infty\ \wedge\ S_{T_i(\tau)+1}=s'\},\\
& N_n(\tau):=\sum_{i=1}^n X_i(\tau).
\end{split}
\end{equation*}
Thus, $X_i$ is the indicator that the $i$th execution of $(s,a)$ transitions to $s'$, and $N_n$ is the total number of such successful transitions to $s'$ across the first $n$ attempts.
(We omit the dependence on $s, a, s', k$ just to keep the notation from being overloaded.)

The \emph{composite goal} $G^{(n)}_{s,a,s',k}$ is the event:
\begin{equation}\label{eq:composite-goal}
\begin{split}
&\Big(A_0=L \ \wedge\ T_i(\tau)<\infty\ \forall i\le n \ \wedge\ N_n(\tau)\le k\Big)\\
&\bigvee\Big(A_0=R \ \wedge\ T_i(\tau)<\infty\ \forall i\le n \ \wedge\ N_n(\tau)> k\Big).
\end{split}
\end{equation}
For convenience, we will write $G^{(n)}_{s,a,s',k} = G^{(n,1)}_{s,a,s',k} \lor G^{(n,2)}_{s,a,s',k}$, where $G^{(n,1)}_{s,a,s',k}$ and $G^{(n,2)}_{s,a,s',k}$ are the first and second disjuncts, respectively.

\textbf{Interpretation:} The goal $G^{(n)}_{s,a,s',k}$ forces a one-shot binary commitment at time $t=0$: choosing $A_0=L$ commits to the branch ``at most $k$ successes'', while choosing $A_0=R$ commits to the branch ``more than $k$ successes''.
After this commitment, the agent must generate $n$ \emph{attempts} to execute $(S_t=s,A_t=a)$; the $i$th such attempt occurs at time $T_i$.
Each attempt counts as a \emph{success} if it transitions to $s'$ on the next step, i.e. $S_{T_i+1}=s'$, and $N_n$ counts the number of successes in $n$ attempts.
Thus $G^{(n)}_{s,a,s',k}$ is an either-or test about whether the transition $(s,a)\to s'$ happens ``rarely'' ($\le k$ times) or ``often'' ($>k$ times) across $n$ attempts.
Equivalently, at $t=0$ the agent chooses between two incompatible branches:
\emph{(i)} ``$\le k$ successes in $n$ attempts of $(s,a)\to s'$'' (signaled by $A_0=L$) or \emph{(ii)} ``$>k$ successes in $n$ attempts'' (signaled by $A_0=R$).
\end{definition}

Next, we deal with the fact that under a stochastic policy, taking either action $L$ or $R$ is actually a mixture of the two.

\begin{lemma}[Binary-decision regret controls wrong-action mass]\label{lem:bet}
Define the \emph{wrong-action mass}
\begin{equation*}
w \;:=\;
\begin{cases}
1-q,&\text{if }u_L\ge u_R \ (\text{$L$ is optimal}),\\
q,&\text{if }u_R>u_L \ (\text{$R$ is optimal}).
\end{cases}
\end{equation*}
Then the normalized regret $\delta$ is equivalent to:
\begin{equation}\label{eq:regret}
\delta \;=\; w\cdot \frac{|u_L-u_R|}{\max\{u_L,u_R\}}.
\end{equation}
In the special \emph{betting} case where $u_L$ and $u_R$ are complementary, namely $u_R:=1-u_L$, defining the margin
$m:=|u_L-\tfrac12|$, we obtain
\begin{equation}
\label{eq:bet-regret}
\delta
= w \cdot \frac{4m}{1+2m}.
\end{equation}
Consequently, on the event $m\ge\gamma\in(0,\tfrac12)$,
\begin{equation}
\label{eq:bet-wa-ub}
w \le \frac{\delta}{c(\gamma)}.
\end{equation}
\end{lemma}

\paragraph{Fully observed diagnostic setup.} For the composite goals $G^{(n)}_{s,a,s',k}$ of Definition~\ref{def:goal}, write \[ q_{s,a,s',k}:= \pi\!\left(G^{(n,1)}_{s,a,s',k}\mid s_0,G^{(n)}_{s,a,s',k}\right), \] and define $V^\pi,V^\star$ and the normalized regret $\delta_{s,a,s',k}(\pi;s_0)$ as in Eq.~\eqref{eq:regret-def}. The induced clipped soft estimator is \begin{equation} \label{eq:p_hat_def} \widehat P_{ss'}(a):= \operatorname{clip}_{[0,1]}\!\left[ \frac1n\left(\sum_{k=0}^{n}(1-q_{s,a,s',k})-\frac12\right) \right]. \end{equation} The clipping only enforces that the estimator is a probability and cannot increase absolute error to the true $P_{ss'}(a)$.
Note this estimator is explicitly \emph{computable} by querying the goal-conditioned policy on each diagnostic goal, recording its probability $q_{s,a,s',k}$ of choosing the first branch, and summing these probabilities as in Eq.~\eqref{eq:p_hat_def}; it does not require estimating transition frequencies from a rollout (though this can be done too).

\begin{theorem}[Fully observed: stochastic policies + average regret $\Rightarrow$ approximate transition model] \label{thm:fo_avg_stoch} 
Under the fully observed diagnostic setup, assume \begin{equation} \label{eq:avg_regret_assumption} \mathbb{E}_{(s,a,s',k)\sim \mathrm{Unif}(\mathcal S\times\mathcal A\times\mathcal S \times\{0,\dots,n\})} \big[\delta_{s,a,s',k}(\pi;s_0)\big] \le \bar\delta . \end{equation} Then, for any fixed $\gamma\in(0,\tfrac12)$, \begin{equation} \label{eq:avg_error_bound} \begin{split} &\mathbb{E}_{(s,a,s')} \Big[\,\big|\widehat P_{ss'}(a)-P_{ss'}(a)\big|\,\Big]\\ \le\;& 2t_\gamma\, \mathbb{E}_{(s,a,s')} \left[ \sqrt{\frac{P_{ss'}(a)(1-P_{ss'}(a))}{n}} \right] + \frac{\bar\delta}{c(\gamma)} + O\!\left(\frac1n\right). \end{split} \end{equation} In particular, \[ \mathbb{E}_{(s,a,s')} \big[|\widehat P_{ss'}(a)-P_{ss'}(a)|\big] \le \frac{t_\gamma}{\sqrt n} + \frac{\bar\delta}{c(\gamma)} + O\!\left(\frac1n\right). \]
\end{theorem}

\begin{remark}[Independence from goal family size]
\citet{richens2025} state a more restricted version of Theorem~\ref{thm:fo_avg_stoch} under a (worst-case) competence assumption over all goals, but note that their proof only needs an explicit diagnostic subset of $O(n|\mathcal A||\mathcal S|^2)$ simple composite goals.
By contrast, our Theorem~\ref{thm:fo_avg_stoch} does not depend on the goal family size because it relaxes the worst-case regret assumption by the average normalized regret assumption \eqref{eq:avg_regret_assumption} on that diagnostic family.
\end{remark}

Notably, the error bound \eqref{eq:avg_error_bound} of Theorem~\ref{thm:fo_avg_stoch} tightens  as the goal depth $n$ increases, reflecting the fact that longer-horizon goal competence forces the agent to estimate transition dynamics with increasing precision. 
In contrast, when $n=1$ (purely myopic goals), accurate world modeling is not required---explicating the classic pitfall behind the Good Regulator Theorem~\citep{conant1970every} that trivial or constant policies can suffice for immediate control, but fail once multi-step coordination is demanded.

A natural question to ask next is under what conditions can we recover a \emph{causal} world model, and of what \emph{type} is the represented causality?

\begin{corollary}[Causal content: approximately recovered interventional kernel]
\label{cor:causal_content_expanded}
Assume the setting and hypotheses of Theorem~\ref{thm:fo_avg_stoch}.
Assume additionally that the controlled Markov process admits an \emph{$\varepsilon_{\mathrm{cMP}}$-approximate} causal Markov-process (cMP) interpretation in which choosing $A_t=a$ corresponds to the intervention $\mathrm{do}(A_t=a)$ and,
for all $s,a,s'$,
\begin{equation}
\label{eq:cMP_identification}
\begin{split}
&\big|P_{ss'}(a) - P^{\mathrm{do}}_{ss'}(a)\big|
\;\le\; \varepsilon_{\mathrm{cMP}},\\
&P^{\mathrm{do}}_{ss'}(a):=P(S_{t+1}=s' \mid S_t=s,\mathrm{do}(A_t=a)).
\end{split}
\end{equation}
Then the estimator $\widehat P$ defined from $\pi$ via \eqref{eq:p_hat_def} satisfies the same average error bound as in Theorem~\ref{thm:fo_avg_stoch}, up to the mismatch $\varepsilon_{\mathrm{cMP}}$: for any fixed $\gamma\in(0,\tfrac12)$,
\begin{equation}
\begin{split}
&\mathbb{E}_{(s,a,s')}\big[\,|\widehat P_{ss'}(a)-P^{\mathrm{do}}_{ss'}(a)|\,\big]\\
\;\le&\; 2t_\gamma\,
\mathbb{E}_{(s,a,s')}
\Big[\sqrt{\tfrac{P_{ss'}(a)(1-P_{ss'}(a))}{n}}\Big]
\;+\;\frac{\bar\delta}{c(\gamma)}\\
&+\;\varepsilon_{\mathrm{cMP}}
\;+\;O\left(\frac{1}{n}\right).
\end{split}
\end{equation}
In particular, low average regret on the diagnostic goal family forces $\pi$ to implicitly approximate Level 2 interventional queries, in the sense of~\citet{pearl2009causality}, of the form $P(S_{t+1}=s' \mid S_t=s,\mathrm{do}(A_t=a))$ up to $\varepsilon_{\mathrm{cMP}}$.
\end{corollary}

Note that Corollary~\ref{cor:causal_content_expanded} does not, in general, identify causal relations between \emph{concurrent} components of the state vector (e.g. between $X_t$ and $Y_t$ when $S_t=(X_t,Y_t)$), since such relations can be non-identifiable from the transition function alone.
It is worth noting that unless the transition function $P$ is a point-mass, namely $S_{t+1} = f(S_t, A_t)$, whereby learning the interventional kernel is exactly equivalent to learning the transition function $f$, then \citet{pearl2009causality} Level 2 of interventions, rather than counterfactuals, is the maximum level of recovery we can guarantee.
This is the same level that \citet[Theorems 1-3]{richens2024robust} reach, but they do it under a much stronger maximum (rather than average) regret assumption under deterministic (rather than stochastic) policies.

In fact, despite generalizing to stochastic policies under average regret, \citet{pearl2009causality} Level 3 (counterfactuals) remains out of reach without additional assumptions:

\begin{corollary}[No generic Level 3 recovery from the interventional kernel]
\label{cor:no_L3}
Even if $\widehat P$ recovers the interventional kernel $P^{\mathrm{do}}_{ss'}(a)$ exactly (in particular, even if $\pi$ is optimal on all the diagnostic goals), the resulting information does not, in general, identify Level $3$ counterfactual queries involving $S_{t+1}^a$ and $S_{t+1}^{a'}$ simultaneously, where $S_{t+1}^a$ denotes the potential next state under $\mathrm{do}(A_t=a)$.
\end{corollary}

Therefore, recovering~\citet{pearl2009causality} Level~3 counterfactuals requires an explicit structural causal model specifying the exogenous noise and its cross-action coupling, not merely the interventional transition kernel $P^{\mathrm{do}}(s'\mid s,a)$.

\section{Selection Theorems under Partial Observability}
\label{sec:wm-po}
Our betting reduction (Lemma~\ref{lem:bet}) also enables selection theorems under \emph{partial} observability, addressing an open question of~\citet{richens2025}. 
The reason this is open, is because under partial observability, we cannot guarantee that the agent's action choices isolate a single underlying transition probability in the way they do in the fully observed case.
When the agent observes only an observation $o_t$ rather than the true state $s_t$, the success probabilities of the diagnostic branches become mixtures over latent states consistent with $o_t$, and different latent dynamics can induce identical observable behavior on all composite goals of bounded depth.
Consequently, low regret does not imply recovery of the underlying transition kernel without additional structure. 
This breaks the direct reduction used in Theorem~\ref{thm:fo_avg_stoch} and requires more careful selection of diagnostic goals defined at the level of predictive beliefs rather than physical states. 
We achieve this by combining our betting reduction from \S\ref{sec:wm-fo} with predictive-state representations (PSRs).

\subsection{Setup and Notation}

\textbf{POMDP.}
A finite partially observed Markov decision process (POMDP) is a tuple
\begin{equation*}
E = (\mathcal X, \mathcal A, \mathcal O, T, Z, \mu_0),
\end{equation*}
where $\mathcal X$ is a finite latent state space, $\mathcal A$ is a finite action space with $|\mathcal A|\ge 2$, $\mathcal O$ is a finite observation space, $T(x'\mid x,a)$ is the transition kernel, $Z(o\mid x)$ is the observation kernel, and $\mu_0\in\Delta(\mathcal X)$ is the initial latent-state distribution.
A history at time $t$ is
\begin{equation*}
h_t:=(o_0,a_0,o_1,\dots,a_{t-1},o_t).
\end{equation*}
For any history $h_t$ and any prescribed future action sequence $A_{t:t+k-1}=\alpha\in\mathcal A^k$, the POMDP induces a well-defined conditional distribution over future observations $O_{t+1:t+k}$.
For convenience, we will drop the subscript $t$ and refer to histories as $h := h_t$.

\textbf{Agent interface (report bit).}
As in the fully observed case, we reduce prediction to a one-shot binary decision.
We allow the agent to emit a report bit $B_t\in\{L,R\}$ that does not affect environment dynamics.
Formally, the agent outputs $(B_t,A_t)\in\{L,R\}\times\mathcal A$, while the environment transition
ignores $B_t$.
This device is without loss of generality for necessity results: any agent can internally commit to one of two incompatible plans before acting, without changing the induced environment process.
All prediction is expressed through the report bit; the environment-action channel is used only to execute prescribed action sequences.

\textbf{Tests (predictive-state style).}
A \emph{test} is a pair
\begin{equation*}
T=(\alpha,W),
\end{equation*}
where $\alpha\in\mathcal A^k$ is a finite action sequence and $W\subseteq\mathcal O^k$ is an event over the resulting observation sequence.
For a history $h$, define the test success probability
\begin{equation*}
p_T(h)
\;:=\;
\Pr\!\big(O_{t+1:t+k}\in W \,\big|\, h,\ A_{t:t+k-1}=\alpha\big),
\end{equation*}
and the associated margin
\begin{equation*}
m_T(h):=\big|p_T(h)-\tfrac12\big|.
\end{equation*}

\textbf{Behavioral distinguishability.}
Two histories $h,h'$ are \emph{behaviorally distinguishable} if there exists a test $T$ with $p_T(h)\neq p_T(h')$.
They are \emph{$\gamma$-distinguishable} if $|p_T(h)-p_T(h')|\ge\gamma$ for some test $T$.
A POMDP is \emph{non-trivially partially observable} if there exist histories with the same last observation that are behaviorally distinguishable.

\textbf{Betting goals induced by tests.}
Each test $T=(\alpha,W)$ induces a one-shot betting goal $g_T$: 
at history $h$, the agent outputs a report bit $B_t$; the environment then executes $A_{t:t+k-1}=\alpha$; the episode succeeds iff $B_t=L$ and $O_{t+1:t+k}\in W$, or $B_t=R$ and $O_{t+1:t+k}\notin W$.
Thus, $g_T$ is a binary bet on whether $W$ occurs under $\alpha$.

\textbf{Policies, value, and regret.}
A (possibly stochastic) goal-conditioned policy specifies $\pi(b\mid h,g_T)$ for $b\in\{L,R\}$.
Let $q_T(h):=\pi(L\mid h,g_T)$.
The success probability under $\pi$ is
\begin{equation}
\label{eq:Vpi}
V^\pi(h;g_T)=q_T(h)\,p_T(h)+(1-q_T(h))(1-p_T(h)),
\end{equation}
while the optimal success probability is
\begin{equation}
\label{eq:Vstar}
V^\star(h;g_T)=\max\{p_T(h),1-p_T(h)\}=\tfrac12+m_T(h).
\end{equation}
Define the normalized regret
\begin{equation*}
\delta_T(\pi;h):=1-\frac{V^\pi(h;g_T)}{V^\star(h;g_T)}\in[0,1].
\end{equation*}

\textbf{Evaluation distribution.}
Let $\mathcal H$ be a distribution over histories and let $D$ be a distribution over tests.
We assume a global average regret bound
\begin{equation}
\label{eq:avg-regret-global}
\mathbb E_{h\sim\mathcal H}\,\mathbb E_{T\sim D}\big[\delta_T(\pi;h)\big]\le\bar\delta.
\end{equation}

\textbf{Wrong-action mass and margins.}
For a test $T$ and history $h$, define the probability mass assigned to the suboptimal bet
\begin{equation*}
w_T(h):=
\begin{cases}
1-q_T(h), & p_T(h)\ge\tfrac12,\\
q_T(h), & p_T(h) < \tfrac12.
\end{cases}
\end{equation*}
For $\gamma\in(0,\tfrac12)$, let
\begin{equation*}
\begin{split}
&E_\gamma:=\{(h,T): m_T(h)\ge\gamma\}\\
&q_\gamma:=\Pr_{h\sim\mathcal H,\,T\sim D}\big((h,T)\in E_\gamma\big).
\end{split}
\end{equation*}

\textbf{Non-degenerate evaluation.}
Our selection results are informative only if the evaluation distribution places nontrivial mass on informative tests.
We assume that for some $\gamma\in(0,\tfrac12)$ there exists a constant $\eta'>0$ such that
\begin{equation*}
%\begin{split}
\Pr\!\big(p_T(h)\ge\tfrac12+\gamma\big)\ge\eta',\qquad \Pr\!\big(p_T(h)\le\tfrac12-\gamma\big)\ge\eta',
%\end{split}
\end{equation*}
thereby implying that $q_\gamma\ge2\eta'$.
These conditions rule out degenerate evaluations where all bets are near coin flips or where one outcome is almost always correct (to avoid the case where a constant policy that always reports $L$ can have very low regret without representing any nontrivial predictive
distinctions, which is a pitfall of the original Good Regulator Theorem~\citep{conant1970every}).

\textbf{Predictive world model.}
In a POMDP, what matters for decision-making is the ability to predict future observations under candidate action sequences.
Accordingly, we use \emph{predictive world model} to mean any internal mechanism sufficient to determine (or approximate) the test probabilities $\{p_T(h)\}$.
In the language of predictive-state representations (PSRs), the vector
\begin{equation*}
\eta_{\mathcal T}(h) := \big(p_T(h)\big)_{T\in\mathcal T}
\end{equation*}
is the \emph{predictive state}.
For sufficiently rich $\mathcal T$, $\eta_{\mathcal T}(h)$ is decision-sufficient; in finite POMDPs, the belief state is \emph{one} such representation~\citep{kaelbling1998}.

\textbf{Memory (representation of history).}
We model the agent’s internal memory abstractly as a representation $M=f(h)$ through which the policy factors:
\begin{equation}\label{eq:memory}
\pi(\cdot\mid h,g_T)=\pi(\cdot\mid M(h),g_T).
\end{equation}
We say that $M$ is \emph{decision-sufficient} for a test family if $M(h)$ determines the optimal bet for all tests in that family, and accordingly that $\pi$ is \textbf{\emph{$\mathbf{M}$-based}}, since $\pi$ depends on $h$ only through $M(h)$ for all betting goals $g_T$.
Our selection theorems show that achieving low average regret on separating betting goals forces the agent's memory to refine the predictive-state partition induced by $\eta_{\mathcal T}$; representations that alias histories with distinct predictive states incur unavoidable regret.

\subsection{Predictive world modeling necessity and recovery under partial observability}
\begin{theorem}[Predictive modeling necessity]\label{thm:predictive}
Fix $\gamma\in(0,\tfrac12)$.
Assume the global average regret bound \eqref{eq:avg-regret-global}.
Then
\begin{equation}
\mathbb E_{h\sim\mathcal H}\ \mathbb E_{T\sim D}\Big[w_T(h)\,\mathbf 1\{m_T(h)\ge\gamma\}\Big]
\;\le\;\frac{\bar\delta}{c(\gamma)}.
\label{eq:thm2-main-global}
\end{equation}
Equivalently, if $q_\gamma>0$ then
\begin{equation}\label{eq:thm2-cond-global}
\mathbb E\Big[w_T(h)\,\big|\, m_T(h)\ge\gamma\Big]\;\le\;\frac{\bar\delta}{q_\gamma\,c(\gamma)}.
\end{equation}
\end{theorem}

In other words, if a policy has small \emph{global average} regret on betting goals, then on tests that are not near a coin-flip ($m_T(h)\ge\gamma$), it must place only small probability mass on the suboptimal bet.
Thus, robust goal performance \emph{selects for} an internal predictive mechanism sufficient to decide many action-conditioned future-observation tests---a minimal, decision-relevant notion of a predictive world model.

However, we may ask what further assumptions we need to recover the predictive state, in an analogous manner to the fully observed case of Theorem~\ref{thm:fo_avg_stoch}, assuming average regret and stochastic policies.

First, we show that recovery is not possible in our current setup with single bets (even under optimal policies), showing that under our assumptions, Theorem~\ref{thm:predictive} is maximally strong:
\begin{proposition}[No generic predictive-state recovery from fair bets]
\label{prop:po-no-recovery}
Even exact optimal query access to the fair betting goals $g_T$ does not, in general, identify the predictive state $\eta_{\mathcal T}(h)$.
Indeed, there exist finite POMDPs $E_p,E_q$ with $|\mathcal X|=4$, a history $h$ with the same last observation in both environments, and parameters $p\neq q$ in $(\tfrac12,1)$ such that for every test $T$ the unique optimal bet for $g_T$ at $h$ is the same in $E_p$ and $E_q$, while $p_T^{E_p}(h)\neq p_T^{E_q}(h)$ for some test $T$.
Consequently, from the family of fair betting decisions alone one cannot, in general, recover the predictive state, and hence not a PSR.
\end{proposition}
This finite-POMDP separation non-vacuously motivates Theorem~\ref{thm:po-threshold-recovery}: identical fair-bet behavior can hide different predictive states, while threshold queries recover their magnitudes.

Next, we show that if we extend the tests to ask the \emph{same} test across \emph{multiple} thresholds, predictive state recovery is possible, as the agent's response curve across thresholds reveals the actual magnitude of $p_T(h)$, and average regret then forces those probabilities to be recoverable:
\paragraph{Threshold-bet setup.}
For a test $T=(\alpha,W)$ and threshold $\lambda\in[0,1]$,
let $g_{T,\lambda}$ be the bet comparing the test success
probability $p_T(h)$ against an independent lottery of success
probability $\lambda$. Write
\[
\begin{aligned}
q_{T,\lambda}(h)&:=\pi(L\mid h,g_{T,\lambda}),\\
V^\pi(h;g_{T,\lambda})
&:=q_{T,\lambda}(h)p_T(h)+(1-q_{T,\lambda}(h))\lambda,\\
V^\star(h;g_{T,\lambda})
&:=\max\{p_T(h),\lambda\},\\
\delta_{T,\lambda}(\pi;h)
&:=1-\frac{V^\pi(h;g_{T,\lambda})}{V^\star(h;g_{T,\lambda})}.
\end{aligned}
\]
For $K\ge1$, let $\lambda_k=(k-\frac12)/K$ and define
\begin{equation}
\label{eq:pt-hat}
\hat p_T(h):=\frac1K\sum_{k=1}^K q_{T,\lambda_k}(h),
\qquad
\varepsilon_K:=2\bar\delta_K+\frac{1}{4K^2}.
\end{equation}

\begin{theorem}[Predictive-state recovery from threshold bets] \label{thm:po-threshold-recovery} Fix $\ell\ge1$ and suppose $D$ is supported on tests $T=(\alpha,W)$ with $|\alpha|\le\ell$. Under the threshold-bet setup, assume \begin{equation} \label{eq:avg-regret-threshold} \mathbb E_{h\sim\mathcal H}\,\mathbb E_{T\sim D} \left[ \frac1K\sum_{k=1}^K\delta_{T,\lambda_k}(\pi;h) \right] \le\bar\delta_K . \end{equation} Then \begin{equation} \label{eq:threshold-recovery-main} \mathbb E_{h\sim\mathcal H}\,\mathbb E_{T\sim D} \Big[ \big(\hat p_T(h)-p_T(h)\big)^2 \Big] \le \varepsilon_K . \end{equation} In particular, if $D$ is uniform over a finite family $\mathcal T_\ell=\{T_1,\dots,T_d\}$ of tests of depth at most $\ell$, and \[ \hat\eta_{\mathcal T_\ell}(h) := \big(\hat p_{T_1}(h),\dots,\hat p_{T_d}(h)\big), \] then \begin{equation} \label{eq:threshold-recovery-vector} \mathbb E_{h\sim\mathcal H} \left[ \frac1d \big\|\hat\eta_{\mathcal T_\ell}(h)-\eta_{\mathcal T_\ell}(h)\big\|_2^2 \right] \le \varepsilon_K . \end{equation} \end{theorem}
Observe that for $K=1$ we recover the counterexample in Proposition~\ref{prop:po-no-recovery} where even for $\bar{\delta}_K = 0$ we get a recovery bound of $1/4$, thereby only giving us information about the \emph{sign} of $p_T(h)$ rather than its underlying value.

The advantage of Theorem~\ref{thm:po-threshold-recovery} is its generality as a recovery method under partial observability, which can be \emph{repeatedly} applied to any history-test pair $(h, T)$, without making any additional assumptions about how the environment dynamics evolve.
This makes it an appealing approach in practice to potentially apply to frontier agents in \emph{open-ended} real-world settings. 
However, it may still be of independent theoretical interest to study under what additional constraints one could recover the explicit compact predictive dynamics operator (the PSR operator) rather than the predictive coordinates $p_T(h)$ on each tested family, which one has to run per test.
Specifically, we show in Theorem~\ref{thm:psr-operator-recovery} that an average-regret recovery is possible under linear finite-dimensional PSR operators, which in practice can hold in restricted, resettable, finite-workflow deployments:
\paragraph{Linear-PSR operator setup.} Let $\mathcal T=\{T_1,\dots,T_d\}$ be a finite core test set. For $\sigma=(a,o)\in\mathcal A\times\mathcal O$ and $T=(\alpha,W)$, write \[ \sigma\circ T:=((a,\alpha),\{o\}\times W). \] Define 
\begin{equation*} 
\begin{split}
&s(h):=(p_{T_1}(h),\dots,p_{T_d}(h)),\\
&s_\sigma(h):=(p_{\sigma\circ T_1}(h),\dots,p_{\sigma\circ T_d}(h)).
\end{split}
\end{equation*}
Assume linear PSR dynamics: for each $\sigma$ there is $B_\sigma\in\mathbb R^{d\times d}$ such that \begin{equation} \label{eq:linear-update} s_\sigma(h)=B_\sigma s(h) \qquad\text{for all histories }h. \end{equation} Choose histories $h^1,\dots,h^d$ such that $S:=[s(h^1)\ \cdots\ s(h^d)]$ is invertible, and set \[ Y_\sigma:=[s_\sigma(h^1)\ \cdots\ s_\sigma(h^d)] =B_\sigma S . \] Using the threshold estimator in Eq.~\eqref{eq:pt-hat}, define $\hat s,\hat s_\sigma,\hat S,\hat Y_\sigma$ analogously.

\begin{theorem}[Linear-PSR operator recovery from threshold bets] \label{thm:psr-operator-recovery} Assume the linear-PSR operator setup and the threshold-bet average-regret bound of Theorem~\ref{thm:po-threshold-recovery} for all tests in \[ \mathcal T\cup \{\sigma\circ T_i:\sigma\in\mathcal A\times\mathcal O,\ i=1,\dots,d\}. \] Then \begin{equation} \label{eq:SY-bound} \|\hat S-S\|_F^2+ \sum_{\sigma\in\mathcal A\times\mathcal O} \|\hat Y_\sigma-Y_\sigma\|_F^2 \le d^2(1+|\mathcal A||\mathcal O|)\varepsilon_K . \end{equation} If additionally \begin{equation} \label{eq:S-cond} d\sqrt{(1+|\mathcal A||\mathcal O|)\varepsilon_K} \le \frac{1}{2\|S^{-1}\|_2}, \end{equation} then $\hat S$ is invertible and, for $\hat B_\sigma:=\hat Y_\sigma\hat S^{-1}$, \begin{equation} \label{eq:B-bound} \sum_\sigma\|\hat B_\sigma-B_\sigma\|_F^2 \le C(S,Y)\varepsilon_K , \end{equation} where \[ C(S,Y):= 8d^2(1+|\mathcal A||\mathcal O|) \left( \|S^{-1}\|_2^2+ \|S^{-1}\|_2^4\sum_\sigma\|Y_\sigma\|_2^2 \right). \] Thus, vanishing average threshold-regret recovers the linear-PSR operators $(B_\sigma)_{\sigma\in\mathcal A\times\mathcal O}$. \end{theorem}

\subsection{Memory necessity}
\paragraph{No-aliasing setup.} Let $M=f(h)$ be any candidate memory statistic, as in Eq.~\eqref{eq:memory}, and let $\mathcal P$ be a distribution over paired histories $(h,h')$ with the same last observation. Define \[ \mathsf{Alias}_M:=\{(h,h'):M(h)=M(h')\}. \] For $\gamma\in(0,\tfrac12)$ and test distribution $D$, assume measurable witness sets $S_\gamma(h,h')$ such that, whenever $T\in S_\gamma(h,h')$, \[ p_T(h)\ge\frac12+\gamma,\qquad p_T(h')\le\frac12-\gamma . \] Define the witnessed aliasing mass and pair-regret \[ q^{\mathsf{Alias}}_\gamma(M):= \Pr_{(h,h')\sim\mathcal P,T\sim D} \big((h,h')\in\mathsf{Alias}_M, T\in S_\gamma(h,h')\big) \] \[ \bar\delta_{\mathcal P}(\pi):= \mathbb E_{(h,h')\sim\mathcal P} \frac12\left( \mathbb E_{T\sim D}[\delta_T(\pi;h)] + \mathbb E_{T\sim D}[\delta_T(\pi;h')] \right). \] All subsequent recoding statements are on the support of $\mathcal P$, not globally over all histories. \begin{theorem}[Memory necessity] \label{thm:memory} Under the no-aliasing setup, any $M$-based policy $\pi$ satisfies \begin{equation} \label{eq:thm1-lb-M} \bar\delta_{\mathcal P}(\pi) \ge q^{\mathsf{Alias}}_\gamma(M)\frac{c(\gamma)}{2}. \end{equation} Consequently, if $\bar\delta_{\mathcal P}(\pi)<q^{\mathsf{Alias}}_\gamma(M)c(\gamma)/2$, then $\pi$ cannot be $M$-based: low pair-regret rules out aliasing histories that induce opposite large-margin bets. \end{theorem}

In other words, if a policy treats two histories the same while the correct bet differs with high confidence, then it must make errors on at least one of them. 
Therefore, low regret rules out memory states that collapse histories needing different confident predictions.

\section{Structured task families: modularity, tradeoffs, and representational match}
\label{sec:structure}
So far for world modeling and memory necessity, we have not introduced major assumptions to the task families we expect the agent to be competent at.
But it turns out that for average-case competence under different task families, we get interesting properties that have to do with the necessity of modularity, tracking internal drives, and inner representational match between agents.
These can be derived very cleanly as corollaries of our previous Theorems~\ref{thm:predictive} and~\ref{thm:memory}, leveraging the same underlying machinery of average-case betting and PSR.
Throughout, we work in the POMDP betting setup of \S\ref{sec:wm-po}, with $\gamma\in(0,\tfrac12)$.

\paragraph{Convention (vanishing regret).}
In what follows, the convention $\bar{\delta}_\mathcal{P} \to 0$ means there exists a \emph{sequence} of admissible policies $(\pi_k)$ under $(\mathcal{P},D)$ with $\bar{\delta}_\mathcal{P}(\pi_k)\to 0$; equivalently, for every $\varepsilon>0$ there exists admissible $\pi$ with $\bar{\delta}_\mathcal{P}(\pi)\le\varepsilon$.

\begin{corollary}[Informational modularity from block-structured tests]
\label{cor:modularity}
Assume $\supp(D)=\bigsqcup_{i=1}^K\mathcal T_i$, with
$p_i:=D(\mathcal T_i)>0$ and $D_i:=D(\cdot\mid T\in\mathcal T_i)$.
For each block $i$, suppose the no-aliasing setup holds with test
distribution $D_i$ and witness sets
$S_{\gamma,i}(h,h')\subseteq\mathcal T_i$. Let
$q^{\mathsf{Alias}}_{\gamma,i}(M)$ denote the corresponding witnessed
aliasing mass, and let $\bar\delta_{\mathcal P}(\pi)$ denote pair-regret
under the original mixture $D$. If $\pi$ is $M$-based, then
\[
q^{\mathsf{Alias}}_{\gamma,i}(M)
\le
\frac{2\,\bar\delta_{\mathcal P}(\pi)}{p_i\,c(\gamma)}
\qquad\text{for every }i.
\]
\newline
Thus, as $\bar\delta_{\mathcal P}(\pi)\to0$, aliasing of $\gamma$-separable pairs vanishes within every block.
\end{corollary}

%Thus, having a modular task distribution to be competent at \emph{necessitates} internal informational modularity in the agent.

\begin{corollary}[Tradeoff/regime tracking from shifting mixtures] \label{cor:mixtures} Let the evaluation draw a latent regime $I\sim\Lambda$ and then $T\sim D_I$, so that the marginal test distribution is $D=\sum_i\Lambda(i)D_i$; the supports of the $D_i$ need not be disjoint. Let $\mathcal P$ be a paired-history distribution with regime labels $I(h)$ and assume the no-aliasing setup holds for $D$ with witnesses $S_\gamma(h,h')$ satisfying \[ T\in S_\gamma(h,h') \implies I(h)\ne I(h'). \] Then any $M$-based policy $\pi$ satisfies
\begin{equation*}
\begin{aligned}
&\Pr_{(h,h')\sim\mathcal P,\ T\sim D}
\!\left(
M(h)=M(h'),\ I(h)\ne I(h'),T\in S_\gamma(h,h')
\right)\\
&\le
\frac{2\,\bar\delta_{\mathcal P}(\pi)}{c(\gamma)}.
\end{aligned}
\end{equation*} 
Thus, as $\bar\delta_{\mathcal P}(\pi)\to0$, memory cannot be insensitive to regime changes that flip a $\gamma$-margin optimal bet for the same queried test. \end{corollary}

Thus, if two regimes can occur under the same last observation and they induce opposite $\gamma$-margin optimal bets for the \emph{same} queried test on nontrivial mass, then low pair-regret on the \emph{same} distribution forces $M(h)$ to distinguish the regime whenever it matters.
More generally, Corollary~\ref{cor:mixtures} implies that competence under mixtures of task distributions provides a normative pressure for maintaining persistent, internal variables that track latent evaluative conditions; in embodied settings, such variables can be viewed as analogous to affective or homeostatic modulators studied in affective neuroscience that globally influence policy, attention, and learning across tasks \citep{ekman1992argument,barrett2017theory}.
Importantly, this is a structural claim about functional organization—global, task-general modulation of behavior under uncertainty---rather than a commitment to any particular theory of emotion or phenomenology.

\begin{corollary}[Representational convergence under $\gamma$-minimality, up to invertible recoding] \label{cor:rep_match} Fix $D$ and $\gamma\in(0,1/2)$. Define the $\gamma$-coarsened decision profile $\ell_D^\gamma(h):= \big(\ell_T^\gamma(h)\big)_{T\in\supp(D)}$, where: \[\ell_T^\gamma(h):= \begin{cases} L,& p_T(h)\ge\tfrac12+\gamma,\\ R,& p_T(h)\le\tfrac12-\gamma,\\ \bot,& \text{otherwise}. \end{cases} \] Let $M_1=f_1(h)$ and $M_2=f_2(h)$ be two memory representations with $M_j$-based policies $\pi_j$. Assume, for $j=1,2$, that $\bar\delta_{\mathcal P}(\pi_j)\to0$, that $M_j$ is $\gamma$-minimal, \[ \ell_D^\gamma(h)=\ell_D^\gamma(h')\implies M_j(h)=M_j(h'), \] and that the witnesses are $\gamma$-complete: for $\mathcal P$-a.e. pair, \[ \ell_D^\gamma(h)\ne\ell_D^\gamma(h') \implies D(S_\gamma(h,h'))>0. \] Then, on the support of $\mathcal P$, each $M_j$ induces exactly the partition given by $\ell_D^\gamma$. Hence $M_1$ and $M_2$ agree up to invertible recoding: there exist measurable maps $\varphi,\psi$ such that almost surely \[ M_1=\varphi(M_2),\qquad M_2=\psi(M_1). \] \end{corollary}

Therefore, under the \emph{same} evaluation family, low pair-regret forces any \emph{sufficient} memory representation to preserve exactly the $\gamma$-margin decision-relevant distinctions between histories; if two agents are also $\gamma$-minimal (no extra splitting beyond those distinctions), then their internal memory states must agree up to a relabeling (invertible recoding) on the evaluation support.

\section{Discussion}
\label{sec:discuss}
This work develops quantitative ``selection theorems''~\citep{wentworth2021selection}: representation-theoretic conclusions derived from performance guarantees. Across fully observed and partially observed settings, we showed that low average-case regret on structured families of action-conditioned prediction tasks \emph{selects for} the predictive internal structure tested by the evaluation family. This yields recovery of the interventional kernel, predictive-state recovery and no-aliasing under partial observability, and further constraints from structured task families: informational modularity, regime-tracking state, and representational convergence up to invertible recoding.

Necessity is task-relative: different diagnostics select different structure. To our knowledge, these are the first quantitative selection theorems linking average-case regret over structured task families to necessary predictive-state and memory structure under partial observability. Unlike classical sufficiency results for belief representations~\citep{sondik1971,kaelbling1998}, our results show that regret-bounded competence alone---without worst-case optimality or determinism---imposes concrete internal constraints. The unifying perspective is that robust competence under uncertainty compresses admissible representations: when the evaluation distribution places mass on large-margin predictive distinctions, aliasing those distinctions incurs constant regret. Thus, predictive state, memory, modular decomposition, and regime-tracking variables are not merely architectural assumptions but \emph{consequences} of task demands.

These results resonate with empirical trends in representation learning and NeuroAI. Increasingly general task demands correlate with increasingly aligned representations across architectures and modalities, including alignment between artificial and biological systems in visual~\citep{yamins2014performance}, auditory~\citep{kell2018task}, motor~\citep{sussillo2015neural}, memory~\citep{nayebi2021explaining}, world-modeling~\citep{nayebi2023neural}, and language~\citep{schrimpf2021neural} brain areas, as well as between \emph{autonomous agents} and \emph{whole-brain data} in larval zebrafish~\citep{keller2025autonomous}. The Contravariance Principle in NeuroAI~\citep{cao2024explanatory} and the Platonic Representation Hypothesis in AI~\citep{huh2024platonic} both hypothesize that general learning pressures drive convergence toward a shared statistical model of reality. Our results provide a complementary formal lens: convergence can arise from \emph{shared} competence constraints and, under minimality, be reversibly mapped across agents as in Corollary~\ref{cor:rep_match}.

As AI systems become increasingly capable, our results suggest that organizational regularities should emerge across architectures: belief-like predictive state, modular specialization, persistent internal state, affective-like regime tracking~\citep{ekman1992argument,barrett2017theory}, and unified predictive representations. These regularities mirror cognitive-architecture themes such as global broadcast and modular processing~\citep{baars1997theater,blum2024ai}, and are relevant to increasingly agentic AI systems~\citep{long2024taking}; not as metaphysical commitments, but as \emph{inevitable} structural consequences of task competence. More empirical evidence is needed for consciousness theories~\citep{cogitate2025adversarial}, so we make no such claims here: subjective experience may depend on \emph{how} these components combine, though behavioral similarity across different brains~\citep{feather2025brain} makes this less likely. Selection theorems thus formally explain how capability constrains internal organization.

\newpage
\begin{acknowledgements}
We thank Lenore Blum, Manuel Blum, Dylan Hadfield-Menell, and Daniel Yamins for helpful discussions, as well as Santiago Cifuentes, Leo Kozachkov, Reece Keller, Noushin Quazi, and the anonymous reviewers for helpful feedback on a draft of this manuscript.
We acknowledge the Burroughs Wellcome Fund (CASI award), Foresight Institute, and Protocol Labs for funding.
\end{acknowledgements}

% References
\bibliography{uai2026-template}

\newpage

\onecolumn

\title{What Capable Agents Must Know: Selection Theorems for Robust Decision-Making under Uncertainty\\(Supplementary Material)}
\maketitle

\appendix
\section{Proof of Lemma~\ref{lem:bet}}
\begin{proof}
Observe that the success probability defined in \eqref{eq:success} can be rewritten as
\begin{equation*}
V = u_R + q(u_L-u_R),
\end{equation*}
which is linear in $q$. 
Thus, the optimal success probability is achieved at the endpoints, $V^\star:=\max_{q \in [0,1]} V = \max\{V(0), V(1)\} = \max\{u_L,u_R\}$. 

Assume wlog $u_L\ge u_R$.
Then $V^\star=u_L$, $w = 1-q$, and $V=(1-w)\,u_L+w\,u_R=u_L-w\,(u_L-u_R)$.
Therefore,
\begin{equation*}
\delta=1-\frac{u_L-w\,(u_L-u_R)}{u_L}=w\frac{u_L-u_R}{u_L}.
\end{equation*}
The other case is symmetric.

In the special case that $u_R = 1-u_L$, then we have that
\begin{equation*}
V^\star = \max\{u_L, u_R\} = \max\{u_L, 1-u_L\} = \tfrac12 + m.
\end{equation*}
Indeed, since $m := \big|u_L - \tfrac12\big|$, we can write $u_L = \tfrac12 + m$ or $u_L = \tfrac12 - m$.
In the first case, $1-u_L = \tfrac12 - m$, and in the second case $1-u_L = \tfrac12 + m$.
In either case, since $m \ge 0$,
\begin{equation*}
\max\{u_L,1-u_L\} = \tfrac12 + m.
\end{equation*}

Moreover, in both cases,
\begin{equation*}
|u_L - u_R|
= |u_L - (1-u_L)|
= |2u_L - 1|
= 2m.
\end{equation*}

Substituting these expressions into \eqref{eq:regret}, yields
\begin{equation}\label{eq:reg-deriv}
\delta
= w \cdot \frac{|u_L-u_R|}{V^\star}
= w \cdot \frac{2m}{\tfrac12 + m}
= w \cdot \frac{4m}{1+2m},
\end{equation}
which proves the stated identity.

Now suppose that $m \ge \gamma > 0$. Since the function
\begin{equation*}
f(m) := \frac{4m}{1+2m}
\end{equation*}
is increasing for $m \ge 0$, since $f'(m) = 4/(1+2m)^2 > 0$, then we have
\begin{equation}\label{eq:ineq}
\frac{4m}{1+2m} \ge \frac{4\gamma}{1+2\gamma} =: c(\gamma).
\end{equation}
Combining \eqref{eq:ineq} with \eqref{eq:reg-deriv} gives
\begin{equation*}
\delta \ge w \, c(\gamma),
\end{equation*}
and hence gives us \eqref{eq:bet-wa-ub}.
\end{proof}

\section{Proof of Theorem~\ref{thm:fo_avg_stoch}}
\begin{proof}
Fix a quadruple $(s,a,s',k)$ and let $X\sim\mathrm{Bin}(n,p)$ with $p:=P_{ss'}(a)$, and define
\begin{equation*}
F(k) := \Pr[X\le k].
\end{equation*}
In what follows, we let $\widehat P_{ss'}(a)$ denote the \emph{unclipped} quantity inside
\eqref{eq:p_hat_def}. 
This is sufficient to upper bound 
$|\widehat P_{ss'}(a)-P_{ss'}(a)|$, since clipping onto $[0,1]$ cannot increase distance to $P_{ss'}(a)\in[0,1]$.

\textbf{1. Pointwise regret lower-bounds wrong-branch mass at margin $m_k$.}
By~\citet[Lemma 6]{richens2025}, the two disjuncts $G^{(n,1)}_{s,a,s',k}$ and $G^{(n,2)}_{s,a,s',k}$ of the composite goal $G^{(n)}_{s,a,s',k}$ have optimal satisfaction probabilities $F(k)$ and $1-F(k)$, respectively:
\begin{equation}\label{eq:v-defs}
\begin{split}
&V^\star(s_0;G^{(n,i)}_{s,a,s',k})
:= \max_{\pi'} \Pr_{\pi'}(G^{(n,i)}_{s,a,s',k}\mid s_0),
\qquad i\in\{1,2\},\\
&V^\star(s_0;G^{(n,1)}_{s,a,s',k}) = F(k),\\
&V^\star(s_0;G^{(n,2)}_{s,a,s',k}) = 1-F(k).
\end{split}
\end{equation}
Hence, by Lemma~\ref{lem:bet}, the optimal satisfaction probability of the \emph{overall} disjunction in \eqref{eq:composite-goal} therefore is
\begin{equation*}
\begin{split}
&V^\star(s_0;G^{(n)}_{s,a,s',k})=\max\{F(k),1-F(k)\}
=\tfrac12+m_k\\
&m_k:=\big|F(k)-\tfrac12\big|,\\
\end{split}
\end{equation*}
Let $w_k$ denote the probability that $\pi$ selects the \emph{suboptimal} branch at threshold $k$:
\begin{equation*}
w_k :=
\begin{cases}
1-q_{s,a,s',k}, & \text{if } V^\star(s_0;G^{(n,1)}_{s,a,s',k}) \ge V^\star(s_0;G^{(n,2)}_{s,a,s',k}),\\
q_{s,a,s',k}, & \text{otherwise},
\end{cases}
\end{equation*}
and let $B_k$ denote the event that $\pi$ selects the disjunct with larger optimal satisfaction probability at threshold $k$.
Then
\begin{equation*}
\Pr_\pi(B_k)=1-w_k
\quad\text{and}\quad
\Pr_\pi(B_k^c)=w_k,
\end{equation*}
since $w_k$ is the probability that $\pi$ selects the suboptimal disjunct.
By the law of total probability,
\begin{equation*}
\begin{split}
&V^\pi(s_0;G^{(n)}_{s,a,s',k})\\
&= \Pr_\pi\!\big(G^{(n)}_{s,a,s',k}\mid B_k\big)\Pr_\pi(B_k)
 + \Pr_\pi\!\big(G^{(n)}_{s,a,s',k}\mid B_k^c\big)\Pr_\pi(B_k^c) \\
&= (1-w_k)\,\Pr_\pi\!\big(G^{(n)}_{s,a,s',k}\mid B_k\big)
 + w_k\,\Pr_\pi\!\big(G^{(n)}_{s,a,s',k}\mid B_k^c\big).
\end{split}
\end{equation*}
Moreover, conditional on $B_k$ (resp.\ $B_k^c$), $\pi$ selects the disjunct with larger (resp. smaller) optimal satisfaction probability, so the corresponding success probability under $\pi$ is at most the optimal success probability of that selected disjunct. 
Hence,
\begin{equation*}
\begin{split}
&V^\pi(s_0;G^{(n)}_{s,a,s',k})\\
&\le (1-w_k)\,\max\!\Big\{V^\star(s_0;G^{(n,1)}_{s,a,s',k}),\,V^\star(s_0;G^{(n,2)}_{s,a,s',k})\Big\} \\
&\quad\;\; + w_k\,\min\!\Big\{V^\star(s_0;G^{(n,1)}_{s,a,s',k}),\,V^\star(s_0;G^{(n,2)}_{s,a,s',k})\Big\} \\
&= (1-w_k)\bigl(\tfrac12+m_k\bigr) + w_k\bigl(\tfrac12-m_k\bigr).
\end{split}
\end{equation*}

Applying Lemma~\ref{lem:bet} with the identification
\begin{equation*}
\begin{split}
&u_L=F(k),\qquad u_R=1-F(k),\\
&m=m_k,\qquad w=w_k,
\end{split}
\end{equation*}
we obtain
\begin{equation*}
\delta_{s,a,s',k}(\pi;s_0) :=
1-\frac{V^\pi}{V^\star}\;\ge\; w_k\,\frac{4m_k}{1+2m_k}.
\end{equation*}
In particular, on the event $\{m_k\ge \gamma\}$,
\begin{equation}
\label{eq:w_le_delta}
w_k \;\le\; \frac{\delta_{s,a,s',k}(\pi;s_0)}{c(\gamma)}.
\end{equation}

\textbf{2. Estimating the binomial median from the policy's disjunct probabilities.}
Let $k_{\mathrm{med}}$ be the (lower) median index of $X$, i.e.
\begin{equation*}
k_{\mathrm{med}} := \min\{k : F(k)\ge \tfrac12\}.
\end{equation*}
By definition of $k_{\mathrm{med}}$, the disjunct with larger optimal satisfaction probability is $G^{(n,2)}_{s,a,s',k}$ for $k<k_{\mathrm{med}}$, and $G^{(n,1)}_{s,a,s',k}$ for $k \ge k_{\mathrm{med}}$.
Define
\begin{equation}\label{eq:k_med_est}
\widehat k_{\mathrm{med}} \;:=\;\sum_{k=0}^n (1-q_{s,a,s',k}),
\end{equation}
which is the expected number of thresholds at which the policy $\pi$ selects the disjunct $G^{(n,2)}_{s,a,s',k}$, by definition of $q_{s,a,s',k}$.
If $\pi$ were optimal on these binary choices, $(1-q_{s,a,s',k})$ would equal $\mathbf 1\{k<k_{\mathrm{med}}\}$, hence $\widehat k_{\mathrm{med}} = k_{\mathrm{med}}$ since $n \ge k_{\mathrm{med}}$. 

In general,
\begin{equation*}
\begin{split}
\big|\widehat k_{\mathrm{med}} - k_{\mathrm{med}}\big| &=
\Big|\sum_{k=0}^n \Big((1-q_{s,a,s',k})-\mathbf 1\{k<k_{\mathrm{med}}\}\Big)\Big|\\
&\le\;\sum_{k=0}^n w_k.
\end{split}
\end{equation*}
Split $\{0,\dots,n\}$ into $K_\gamma:=\{k: m_k<\gamma\}$ and its complement $K_\gamma^c:=\{k: m_k \ge \gamma\}$.
Using the trivial bound $w_k \le 1$ on $K_\gamma$, where $m_k < \gamma$ and the disjuncts become indistinguishable as $m_k \to 0$, so regret cannot constrain the policy’s choice, and the regret-based bound \eqref{eq:w_le_delta} on $K_\gamma^c$,
\begin{equation}
\label{eq:median_error_split}
\begin{split}
\big|\widehat k_{\mathrm{med}}-k_{\mathrm{med}}\big| &\;\le\;
|K_\gamma| \;+\; \frac{1}{c(\gamma)}\sum_{k\notin K_\gamma}\delta_{s,a,s',k}(\pi;s_0)\\
&\;\le\;
|K_\gamma| \;+\; \frac{1}{c(\gamma)}\sum_{k=0}^n\delta_{s,a,s',k}(\pi;s_0).
\end{split}
\end{equation}

\textbf{3. Controlling $|K_\gamma|$.}
Let $\mu:=\mathbb E[X]=np$ and $\sigma^2:=\mathrm{Var}(X)=np(1-p)$.
For any $t>0$, if $k\ge \mu+t\sigma$, then by the one-sided Chebyshev inequality,
\begin{equation*}
\Pr[X\ge k]
= \Pr[X-\mu\ge t\sigma]
\;\le\;\frac{1}{1+t^2},
\end{equation*}
and hence
\begin{equation*}
F(k)=\Pr[X\le k]\;\ge\;1-\frac{1}{1+t^2}=\frac{t^2}{1+t^2}.
\end{equation*}
If $t \ge 1$, then $F(k)\ge \tfrac12$ in this regime. 
Thus, it follows that
\begin{equation}\label{eq:m_k}
m_k = F(k)-\tfrac12 \;\ge\; \frac{t^2}{1+t^2}-\tfrac12
= \frac{t^2-1}{2(1+t^2)}.
\end{equation}
Choosing
\begin{equation*}
t_\gamma := \sqrt{\frac{1+2\gamma}{1-2\gamma}} \;\ge 1,
\end{equation*}
since $\gamma > 0$, makes the right-hand side of \eqref{eq:m_k} equal to $\gamma$, so $m_k\ge \gamma$ whenever $k\ge \mu+t_\gamma\sigma$.

By symmetry, applying the same argument to $-X$ yields that $m_k\ge \gamma$ whenever $k\le \mu-t_\gamma\sigma$.
Equivalently, by contraposition,
\begin{equation*}
k\in K_\gamma=\{k:m_k<\gamma\}
\quad\Longrightarrow\quad
|k-\mu|<t_\gamma\sigma.
\end{equation*}
Therefore,
\begin{equation*}
K_\gamma \;\subseteq\; (\mu-t_\gamma\sigma,\;\mu+t_\gamma\sigma),
\end{equation*}
and since $k$ ranges over integers, the number of such indices is bounded by
\begin{equation}\label{eq:Kgamma_size}
\begin{split}
|K_\gamma| &\;\le\;
\big\lceil 2t_\gamma\sigma \big\rceil + 1\\
&\;\le\; 2t_\gamma\sigma + 2 \\
& = 2t_\gamma\sqrt{np(1-p)}+2.
\end{split}
\end{equation}
Note that this Chebyshev step is deliberately distribution-free and can be loose for small $n$; sharper binomial concentration would improve constants without changing the selection argument.

\textbf{4. From median error to transition-probability error.}
Define $\widehat p:=\widehat P_{ss'}(a)$ as in \eqref{eq:p_hat_def}, i.e. $\widehat p = \frac{1}{n}(\widehat k_{\mathrm{med}}-\tfrac12)$ by definition \eqref{eq:k_med_est}.
A standard binomial fact is that the median differs from the mean by at most $1$: $|k_{\mathrm{med}}-np|\le 1$; equivalently, $k_{\mathrm{med}}\in\{\lfloor np\rfloor,\lceil np\rceil\}$.
Hence,
\begin{equation*}
\begin{split}
|\widehat p-p| &\;\le\;
\Big|\widehat p-\frac{k_{\mathrm{med}}}{n}\Big|+\Big|\frac{k_{\mathrm{med}}}{n}-p\Big|\\
&\;\le\;
\frac{|\widehat k_{\mathrm{med}}-k_{\mathrm{med}}|+\tfrac12}{n}+\frac{1}{n}\\
&\;\le\; \frac{|\widehat k_{\mathrm{med}}-k_{\mathrm{med}}|}{n}+O\left(\frac1n\right).
\end{split}
\end{equation*}
Combining \eqref{eq:median_error_split} and \eqref{eq:Kgamma_size} gives, for this fixed $(s,a,s')$,
\begin{equation*}
\begin{split}
|\widehat p-p| \;\le
&\;2t_\gamma\sqrt{\frac{p(1-p)}{n}}
\;+\;
\frac{1}{nc(\gamma)}\sum_{k=0}^n\delta_{s,a,s',k}(\pi;s_0)\\
&\;+\; O\left(\frac1n\right).
\end{split}
\end{equation*}
Finally, average over $(s,a,s')\sim\mathrm{Unif}(\mathcal S\times\mathcal A\times\mathcal S)$ and use the global assumption \eqref{eq:avg_regret_assumption} to bound
\begin{equation*}
\begin{split}
&\mathbb E_{(s,a,s')}\Big[\frac{1}{n}\sum_{k=0}^n\delta_{s,a,s',k}(\pi;s_0)\Big] \\
&= \frac{n+1}{n}\,
\mathbb E_{(s,a,s',k)}[\delta_{s,a,s',k}(\pi;s_0)]\\
&\;\le\;\frac{n+1}{n}\bar\delta,
\end{split}
\end{equation*}
which yields \eqref{eq:avg_error_bound}.
\end{proof}

\section{Proof of Corollaries~\ref{cor:causal_content_expanded} and~\ref{cor:no_L3}}
\subsection{Proof of Corollary~\ref{cor:causal_content_expanded}}
\begin{proof}
By the \emph{$\varepsilon_{\mathrm{cMP}}$-approximate} causal Markov-process assumption \eqref{eq:cMP_identification}, for each $(s,a,s')$ we have $|P_{ss'}(a)-P^{\mathrm{do}}_{ss'}(a)|\le \varepsilon_{\mathrm{cMP}}$.
Thus, by the triangle inequality,
\begin{equation*}
\begin{split}
|\widehat P_{ss'}(a)-P^{\mathrm{do}}_{ss'}(a)|
&\le
|\widehat P_{ss'}(a)-P_{ss'}(a)|
+|P_{ss'}(a)-P^{\mathrm{do}}_{ss'}(a)|\\
&\le |\widehat P_{ss'}(a)-P_{ss'}(a)|
+\varepsilon_{\mathrm{cMP}}.
\end{split}
\end{equation*}
Taking expectations over $(s,a,s')\sim\mathrm{Unif}(\mathcal S\times\mathcal A\times\mathcal S)$, the bound follows immediately from Theorem~\ref{thm:fo_avg_stoch} by substitution.
\end{proof}

\subsection{Proof of Corollary~\ref{cor:no_L3}}
\begin{proof}
Two structural causal models can share the same interventional kernel 
$P(S_{t+1}\mid S_t=s,\mathrm{do}(A_t=a))$ for all $(s,a)$ 
while differing in counterfactual couplings. 

Fix a single state $s$, binary actions $\{0,1\}$, and binary next state.
Let $U\sim\mathrm{Bernoulli}(1/2)$.  
Model (I): $S_{t+1}=U$.  
Model (II): $S_{t+1}=A_t\oplus U$.  

Both satisfy $P(S_{t+1}=1\mid S_t=s,\mathrm{do}(A_t=a))=1/2$ for $a\in\{0,1\}$, so their interventional kernels coincide. 
However, conditioning on $A_t=0$ and $S_{t+1}=1$ (hence $U=1$), the counterfactual under $A_t=1$ gives $S_{t+1}^{1}=1$ in Model (I) but $S_{t+1}^{1}=0$ in Model (II). 
Thus Level~3 counterfactuals are not identified by the interventional kernel.
\end{proof}

\section{Proof of Theorem~\ref{thm:predictive}}
\begin{proof}
Fix $(h,T)$ and write $p:=p_T(h)$, $m:=m_T(h)$, $q:=q_T(h)$.
If $p\ge\tfrac12$ then $p=\tfrac12+m$ and the suboptimal mass is $w_T(h)=1-q$.
Using \eqref{eq:Vpi}-\eqref{eq:Vstar}:
\begin{equation*}
\begin{split}
&V^\pi(h;g_T)=q(\tfrac12+m)+(1-q)(\tfrac12-m)=\tfrac12-m+2mq,\\
&V^\star(h;g_T)=\tfrac12+m.
\end{split}
\end{equation*}
Thus
\begin{equation}\label{eq:pointwise}
\begin{split}
\delta_T(\pi;h)
=1-\frac{\tfrac12-m+2mq}{\tfrac12+m}
&=\frac{2m(1-q)}{\tfrac12+m}\\
&=w_T(h)\frac{4m}{1+2m}.
\end{split}
\end{equation}
If $p < \tfrac12$, the symmetric calculation yields the same identity $\delta_T(\pi;h)=w_T(h)\cdot\frac{4m}{1+2m}$ (now $w_T(h)=q$).
Therefore, on the event $\{m\ge\gamma\}$ we have
\begin{equation*}
\delta_T(\pi;h)\;\ge\;w_T(h)\cdot \frac{4\gamma}{1+2\gamma}\;=\;w_T(h)\,c(\gamma).
\end{equation*}
Taking expectations over $h\sim\mathcal H$ and $T\sim D$ and using \eqref{eq:avg-regret-global}:
\begin{equation*}
\begin{split}
\bar\delta \ge\ \mathbb E[\delta_T(\pi;h)]
&\ge \mathbb E[\delta_T(\pi;h)\,\mathbf 1\{m_T(h)\ge\gamma\}]\\
&\ge\ c(\gamma)\,\mathbb E[w_T(h)\,\mathbf 1\{m_T(h)\ge\gamma\}],
\end{split}
\end{equation*}
which proves \eqref{eq:thm2-main-global}.
The conditional bound \eqref{eq:thm2-cond-global} follows by dividing by $q_\gamma=\Pr(m_T(H)\ge\gamma)$.
\end{proof}

\section{Proof of Proposition~\ref{prop:po-no-recovery}}

\begin{proof}
Fix any $p,q\in(\tfrac12,1)$ with $p\neq q$.
Let $\mathcal A$ be any finite action space with $|\mathcal A|\ge 2$.
For each $r\in\{p,q\}$, define a POMDP $E_r=(\mathcal X,\mathcal A,\mathcal O,T,Z,\mu_0)$ as follows:
\begin{equation*}
\mathcal X=\{x_0,x_1,y_0,y_1\},
\qquad
\mathcal O=\{u,0,1\},
\end{equation*}
with initial distribution
\begin{equation*}
\mu_0(x_0)=r,
\qquad
\mu_0(x_1)=1-r,
\qquad
\mu_0(y_0)=\mu_0(y_1)=0.
\end{equation*}
The observation kernel is
\begin{equation*}
Z(u\mid x_0)=Z(u\mid x_1)=1,
\qquad
Z(0\mid y_0)=1,
\qquad
Z(1\mid y_1)=1.
\end{equation*}
For every action $a\in\mathcal A$, the transition kernel is
\begin{equation*}
T(y_0\mid x_0,a)=1,
\qquad
T(y_1\mid x_1,a)=1,
\qquad
T(y_0\mid y_0,a)=1,
\qquad
T(y_1\mid y_1,a)=1.
\end{equation*}

Let $h=(u)$ be the initial history.
Fix any test $T=(\alpha,W)$ with $|\alpha|=k$.
Since actions do not affect the dynamics, conditional on $h$ the future observation sequence is deterministically either $0^k$ or $1^k$.
Therefore
\begin{equation}
\label{eq:po-counterexample-pt}
p_T^{E_r}(h)
=
r\,\mathbf 1\{0^k\in W\}
+
(1-r)\,\mathbf 1\{1^k\in W\}.
\end{equation}
Hence $p_T^{E_r}(h)$ can only take one of the four values
\begin{equation*}
0,\qquad 1-r,\qquad r,\qquad 1.
\end{equation*}
Since $r>\tfrac12$, the unique optimal fair bet is:
\begin{itemize}
\item report $L$ if $0^k\in W$;
\item report $R$ if $0^k\notin W$.
\end{itemize}
This rule is independent of the value of $r\in(\tfrac12,1)$.
Therefore every optimal policy for the fair betting goals $g_T$ induces exactly the same answers on all tests at $h$ in $E_p$ and $E_q$.

On the other hand, if $T^\star=(\alpha^\star,\{0^{|\alpha^\star|}\})$ for any fixed action sequence $\alpha^\star$, then by \eqref{eq:po-counterexample-pt},
\begin{equation*}
p_{T^\star}^{E_p}(h)=p
\qquad\text{and}\qquad
p_{T^\star}^{E_q}(h)=q,
\end{equation*}
so $p_{T^\star}^{E_p}(h)\neq p_{T^\star}^{E_q}(h)$.
Thus
\begin{equation*}
\eta_{\mathcal T}^{E_p}(h)\neq \eta_{\mathcal T}^{E_q}(h).
\end{equation*}
This proves that exact optimal query access to the fair betting goals does not, in general, identify the predictive state.
\end{proof}

\section{Proof of Theorem~\ref{thm:po-threshold-recovery}}

\begin{proof}
Fix $(h,T)$ and write
\begin{equation*}
p:=p_T(h),
\qquad
q_k:=q_{T,\lambda_k}(h),
\qquad
\hat p:=\frac1K\sum_{k=1}^K q_k.
\end{equation*}
Let
\begin{equation*}
r_k
:=
V^\star(h;g_{T,\lambda_k})-V^\pi(h;g_{T,\lambda_k})
\end{equation*}
denote the unnormalized regret at threshold $\lambda_k$.
Then
\begin{equation}
\label{eq:rk-threshold}
r_k=
\begin{cases}
(1-q_k)(p-\lambda_k), & \lambda_k\le p,\\[4pt]
q_k(\lambda_k-p), & \lambda_k>p.
\end{cases}
\end{equation}

Let
\begin{equation*}
J:=\#\{k:\lambda_k\le p\}=\Big\lfloor Kp+\frac12\Big\rfloor.
\end{equation*}
Averaging \eqref{eq:rk-threshold} over $k$ gives
\begin{equation}
\label{eq:R-threshold}
R
:=
\frac1K\sum_{k=1}^K r_k
=
\frac1K\sum_{k\le J}(p-\lambda_k)-p\hat p+\frac1K\sum_{k=1}^K \lambda_k q_k.
\end{equation}

We now bound the two terms on the right-hand side of \eqref{eq:R-threshold}.

First,
\begin{equation*}
\frac1K\sum_{k\le J}(p-\lambda_k)
=
\frac{Jp}{K}
-
\frac1K\sum_{k=1}^J \frac{k-\tfrac12}{K}
=
\frac{Jp}{K}-\frac{J^2}{2K^2}.
\end{equation*}
Since
\begin{equation*}
\frac{Jp}{K}-\frac{J^2}{2K^2}
=
\frac{p^2}{2}
-
\frac{(Kp-J)^2}{2K^2},
\end{equation*}
and $|Kp-J|\le \tfrac12$, it follows that
\begin{equation}
\label{eq:first-term-threshold}
\frac1K\sum_{k\le J}(p-\lambda_k)
\ge
\frac{p^2}{2}-\frac{1}{8K^2}.
\end{equation}

Second, among all choices $q_k\in[0,1]$ with fixed average $\hat p$, the weighted sum $\sum_k \lambda_k q_k$ is minimized by placing as much mass as possible on the smallest thresholds.
Let $L=\lfloor K\hat p \rfloor$ and $\beta = K\hat p - L \in [0,1]$. 
The minimum is attained by
\begin{equation*}
q_1=\cdots=q_L=1,
\qquad
q_{L+1}=\beta,
\qquad
q_{L+2}=\cdots=q_K=0,
\end{equation*}
(with the convention that if $\beta=0$ the partial entry is omitted), and equals
\begin{align}
\frac1K\sum_{k=1}^K \lambda_k q_k
&\ge
\frac1K
\left(
\sum_{k=1}^L \frac{k-\tfrac12}{K}
+
\beta\frac{L+\tfrac12}{K}
\right) \notag\\
&=
\frac{(L+\beta)^2+\beta(1-\beta)}{2K^2}
\notag\\
&\ge
\frac{\hat p^2}{2}.
\label{eq:second-term-threshold}
\end{align}

Substituting \eqref{eq:first-term-threshold} and \eqref{eq:second-term-threshold} into \eqref{eq:R-threshold} yields
\begin{equation}
\label{eq:R-lower-threshold}
R
\ge
\frac{p^2}{2}-\frac{1}{8K^2}-p\hat p+\frac{\hat p^2}{2}
=
\frac{(\hat p-p)^2}{2}-\frac{1}{8K^2}.
\end{equation}

Now
\begin{equation*}
r_k
=
\delta_{T,\lambda_k}(\pi;h)\,V^\star(h;g_{T,\lambda_k})
\le
\delta_{T,\lambda_k}(\pi;h),
\end{equation*}
since $V^\star(h;g_{T,\lambda_k})\le 1$.
Therefore
\begin{equation}
\label{eq:R-upper-threshold}
R
\le
\frac1K\sum_{k=1}^K \delta_{T,\lambda_k}(\pi;h).
\end{equation}
Combining \eqref{eq:R-lower-threshold} and \eqref{eq:R-upper-threshold},
\begin{equation}
\label{eq:pointwise-threshold-bound}
\big(\hat p-p_T(h)\big)^2
\le
2\Big(
\frac1K\sum_{k=1}^K \delta_{T,\lambda_k}(\pi;h)
\Big)
+
\frac{1}{4K^2}.
\end{equation}

Finally, average \eqref{eq:pointwise-threshold-bound} over $h\sim\mathcal H$ and $T\sim D$, and use \eqref{eq:avg-regret-threshold}, obtaining
\begin{equation*}
\mathbb E_{h\sim\mathcal H}\,\mathbb E_{T\sim D}
\Big[
\big(\hat p_T(h)-p_T(h)\big)^2
\Big]
\le
2\bar\delta_K+\frac{1}{4K^2},
\end{equation*}
which proves \eqref{eq:threshold-recovery-main}.

For the vector statement, if $D$ is uniform on $\mathcal T_\ell=\{T_1,\dots,T_d\}$ then
\begin{align*}
\mathbb E_{h\sim\mathcal H}
\Big[
\frac1d \big\|\hat\eta_{\mathcal T_\ell}(h)-\eta_{\mathcal T_\ell}(h)\big\|_2^2
\Big]
&=
\mathbb E_{h\sim\mathcal H}
\Big[
\frac1d\sum_{j=1}^d
\big(\hat p_{T_j}(h)-p_{T_j}(h)\big)^2
\Big]\\
&=
\mathbb E_{h\sim\mathcal H}\,\mathbb E_{T\sim D}
\Big[
\big(\hat p_T(h)-p_T(h)\big)^2
\Big]\\
&\le
2\bar\delta_K+\frac{1}{4K^2},
\end{align*}
proving \eqref{eq:threshold-recovery-vector}.
\end{proof}

\section{Proof of Theorem~\ref{thm:psr-operator-recovery}}
\begin{proof}
Apply Theorem~\ref{thm:po-threshold-recovery} with $\mathcal{H}$ uniform on $\{h^1,\dots,h^d\}$ and with the test family defined as the indexed collection
\begin{equation*}
\mathcal T^+_{\mathrm{idx}} := \{T_1,\dots,T_d\} \;\cup\;\{\sigma\circ T_j : \sigma\in\mathcal A\times\mathcal O,\; 1\le j\le d\},
\end{equation*}
counting each pair $(\sigma,j)$ separately. 
Then
\begin{equation*}
|\mathcal T^+_{\mathrm{idx}}| = d\bigl(1+|\mathcal A||\mathcal O|\bigr).
\end{equation*}

By \eqref{eq:threshold-recovery-main},
\begin{equation}
\label{eq:avg-squared}
\frac{1}{d\,|\mathcal T^+_{\mathrm{idx}}|}\sum_{i=1}^d \sum_{T\in \mathcal T^+_{\mathrm{idx}}}\left(\hat p_T(h^i)-p_T(h^i)\right)^2\;\le\;\varepsilon_K.
\end{equation}

By construction, the sum over $T\in\mathcal T^+_{\mathrm{idx}}$ exactly enumerates all entries of $S$ (from $T_j$) and all entries of each $Y_\sigma$ (from $\sigma\circ T_j$), so the left-hand side of \eqref{eq:avg-squared} equals
\begin{equation*}
\frac{1}{d^2(1+|\mathcal A||\mathcal O|)}\left(\|\hat S - S\|_F^2 + \sum_{\sigma}\|\hat Y_\sigma - Y_\sigma\|_F^2\right),
\end{equation*}
which proves \eqref{eq:SY-bound}.

Let $\kappa := \|S^{-1}\|_2$. 
From \eqref{eq:SY-bound},
\begin{equation*}
\|\hat S - S\|_2 \le \|\hat S - S\|_F \le d\sqrt{(1+|\mathcal A||\mathcal O|)\,\varepsilon_K}.
\end{equation*}
Under \eqref{eq:S-cond}, this implies
\begin{equation*}
\|S^{-1}(\hat S - S)\|_2 \le \frac{1}{2}.
\end{equation*}
Hence $\hat S$ is invertible, and the standard Neumann series perturbation bound gives
\begin{equation}
\label{eq:Sinv-perturb}
\|\hat S^{-1}\|_2 \le 2\kappa,
\qquad
\|\hat S^{-1} - S^{-1}\|_2 \le 2\kappa^2 \|\hat S - S\|_F.
\end{equation}

For each $\sigma$,
\begin{equation*}
\hat B_\sigma - B_\sigma = (\hat Y_\sigma - Y_\sigma)\hat S^{-1} + Y_\sigma(\hat S^{-1} - S^{-1}).
\end{equation*}
Using \eqref{eq:Sinv-perturb},
\begin{equation*}
\|\hat B_\sigma - B_\sigma\|_F \le
2\kappa\,\|\hat Y_\sigma - Y_\sigma\|_F + 2\kappa^2\,\|Y_\sigma\|_2\,\|\hat S - S\|_F.
\end{equation*}
Squaring and using $(u+v)^2 \le 2u^2 + 2v^2$,
\begin{equation*}
\|\hat B_\sigma - B_\sigma\|_F^2 \le 8\kappa^2 \|\hat Y_\sigma - Y_\sigma\|_F^2 + 8\kappa^4 \|Y_\sigma\|_2^2 \|\hat S - S\|_F^2.
\end{equation*}

Summing over $\sigma$ and applying \eqref{eq:SY-bound} yields \eqref{eq:B-bound}.
\end{proof}

\section{Proof of Theorem~\ref{thm:memory}}
\begin{proof}
Fix any $M$-based policy $\pi$. 
Consider any pair $(h,h')$ with $(h,h')\in\mathsf{Alias}_M$.
Because $\pi$ is $M$-based and $M(h)=M(h')$, for every test $T$ we have identical bet distributions: $q_T(h)=q_T(h')=:q_T$.

Now fix a test $T\in S_\gamma(h,h')$. 
By assumption, $p_T(h)\ge \tfrac12+\gamma$ so the optimal bet at $h$ is $L$ and thus $w_T(h)=1-q_T$; while $p_T(h')\le \tfrac12-\gamma$ so the optimal bet at $h'$ is $R$ and thus $w_T(h')=q_T$.
Therefore,
\begin{equation*}
\frac12\big(w_T(h)+w_T(h')\big) \;=\; \frac12.
\end{equation*}
By the pointwise identity \eqref{eq:pointwise} from the proof of Theorem~\ref{thm:predictive}, when $m_T(\cdot)\ge\gamma$ we have
$\delta_T(\pi;\cdot)\ge c(\gamma)\,w_T(\cdot)$. 
Hence for $T\in S_\gamma(h,h')$,
\begin{equation*}
\begin{split}
&\frac12\big(\delta_T(\pi;h)+\delta_T(\pi;h')\big)\\
&\ge \frac12\,c(\gamma)\big(w_T(h)+w_T(h')\big)
\ =\ \frac{c(\gamma)}{2}.
\end{split}
\end{equation*}
Taking expectations over $(h,h')\sim\mathcal P$ and $T\sim D$ and restricting to the event $\{(h,h')\in\mathsf{Alias}_M,\ T\in S_\gamma(h,h')\}$ yields
\begin{equation*}
\bar\delta_{\mathcal P}(\pi)
\ \ge\ q^{\mathsf{Alias}}_\gamma(M)\cdot \frac{c(\gamma)}{2},
\end{equation*}
which proves \eqref{eq:thm1-lb-M}.
\end{proof}

\section{Proof of Corollary~\ref{cor:modularity}}
\begin{proof}
For each $i$, define the \emph{blockwise} pair-averaged regret
\begin{equation*}
\bar\delta_{\mathcal P,i}(\pi)
:=
\mathbb E_{(h,h')\sim\mathcal P}\ \frac12\Big(\mathbb E_{T\sim D_i}[\delta_T(\pi;h)]
+\mathbb E_{T\sim D_i}[\delta_T(\pi;h')]\Big).
\end{equation*}
Applying Theorem~\ref{thm:memory} with test distribution $D_i$ yields
\begin{equation*}
\bar\delta_{\mathcal P,i}(\pi)\ \ge\ q^{\mathsf{Alias}}_{\gamma,i}(M) \frac{c(\gamma)}{2},
\end{equation*}
so 
\begin{equation}\label{eq:ast}
q^{\mathsf{Alias}}_{\gamma,i}(M) \le \frac{2\,\bar\delta_{\mathcal P,i}(\pi)}{c(\gamma)}.
\end{equation}

Now relate $\bar\delta_{\mathcal P}(\pi)$ (under $D$) to the $\bar\delta_{\mathcal P,i}(\pi)$.
Because $D=\sum_{i=1}^K p_i D_i$, for any fixed history $h$ we have
\begin{equation*}
\mathbb E_{T\sim D}[\delta_T(\pi;h)]
=\sum_{i=1}^K p_i\,\mathbb E_{T\sim D_i}[\delta_T(\pi;h)].
\end{equation*}
Substituting this identity into the definition of $\bar\delta_{\mathcal P}(\pi)$ and exchanging sums/expectations gives
\begin{equation*}
\bar\delta_{\mathcal P}(\pi)
=\sum_{i=1}^K p_i\,\bar\delta_{\mathcal P,i}(\pi).
\end{equation*}
Since all terms are nonnegative, $\bar\delta_{\mathcal P}(\pi)\ge p_i\,\bar\delta_{\mathcal P,i}(\pi)$, hence
\begin{equation}\label{eq:dag}
\bar\delta_{\mathcal P,i}(\pi)\ \le\ \frac{\bar\delta_{\mathcal P}(\pi)}{p_i}.
\end{equation}
Combining \eqref{eq:ast} and \eqref{eq:dag} yields
\begin{equation*}
q^{\mathsf{Alias}}_{\gamma,i}(M)
\ \le\ \frac{2}{c(\gamma)}\cdot \frac{\bar\delta_{\mathcal P}(\pi)}{p_i},
\end{equation*}
as claimed.
\end{proof}

\section{Proof of Corollaries~\ref{cor:mixtures} and~\ref{cor:rep_match}}
For the next two corollaries, it will be useful to have the following lemma:
\begin{lemma}[Low pair-regret $\Rightarrow$ small aliasing mass on $\gamma$-separations]
\label{lem:alias_bound}
Work in the setting of Theorem~\ref{thm:memory} with $(\mathcal P,D,\gamma)$ and witness sets $S_\gamma(h,h')$.
Let $\pi$ be $M$-based.
Then
\begin{equation*}
\begin{split}
&\Pr_{(h,h')\sim\mathcal P,\ T\sim D}\Big(M(h)=M(h')\ \wedge\ T\in S_\gamma(h,h')\Big)\\ 
& \le \frac{2\,\bar\delta_{\mathcal P}(\pi)}{c(\gamma)}.
\end{split}
\end{equation*}
Equivalently,
\begin{equation*}
\mathbb E_{(h,h')\sim\mathcal P}\Big[\mathbf 1\{M(h)=M(h')\}\cdot D(S_\gamma(h,h'))\Big]
\ \le\ \frac{2\,\bar\delta_{\mathcal P}(\pi)}{c(\gamma)}.
\end{equation*}
In particular, if $\bar\delta_{\mathcal P}(\pi)=0$ then $\mathbf 1\{M(h)=M(h')\}\cdot D(S_\gamma(h,h'))=0$ for $\mathcal P$ almost everywhere on $(h,h')$.
\end{lemma}
\begin{proof}
Theorem~\ref{thm:memory} gives the lower bound
\begin{equation*}
\bar\delta_{\mathcal P}(\pi)\ \ge\ q^{\mathsf{Alias}}_\gamma(M)\cdot \frac{c(\gamma)}{2},
\end{equation*}
where
\begin{equation*}
\begin{split}
&q^{\mathsf{Alias}}_\gamma(M)\\
&=\Pr_{(h,h')\sim\mathcal P,\ T\sim D}\Big(M(h)=M(h')\ \wedge\ T\in S_\gamma(h,h')\Big).
\end{split}
\end{equation*}
Rearranging yields the first inequality. 
For the second display, note that
\begin{equation*}
\begin{split}
&q^{\mathsf{Alias}}_\gamma(M)\\
&=\mathbb E_{(h,h')\sim\mathcal P}\Big[\mathbf 1\{M(h)=M(h')\}\cdot D(S_\gamma(h,h'))\Big],
\end{split}
\end{equation*}
by the law of total expectation over $T\sim D$.
The final claim follows because a nonnegative random variable with zero expectation is zero almost surely.
\end{proof}

\subsection{Proof of Corollary~\ref{cor:mixtures}}
\begin{proof}
By assumption, the witness set is supported on regime-mismatched pairs in the sense that
\begin{equation*}
T\in S_\gamma(h,h')\ \Longrightarrow\ I(h)\neq I(h').
\end{equation*}
Therefore the event in the corollary simplifies:
\begin{equation*}
\begin{split}
&\{M(h)=M(h')\ \wedge\ I(h)\neq I(h')\ \wedge\ T\in S_\gamma(h,h')\}\\
&=\{M(h)=M(h')\ \wedge\ T\in S_\gamma(h,h')\},
\end{split}
\end{equation*}
since $T\in S_\gamma(h,h')$ already implies $I(h)\neq I(h')$.
Taking probabilities under $(h,h')\sim\mathcal P$ and $T\sim D$ yields
\begin{equation*}
\begin{split}
&\Pr\big(M(h)=M(h')\ \wedge\ I(h)\neq I(h')\ \wedge\ T\in S_\gamma(h,h')\big)\\
&= \Pr\big(M(h)=M(h')\ \wedge\ T\in S_\gamma(h,h')\big).
\end{split}
\end{equation*}
Now apply Lemma~\ref{lem:alias_bound} with the same $(\mathcal P,D,\gamma,S_\gamma)$ to obtain
\begin{equation*}
\begin{split}
&\Pr_{(h,h')\sim\mathcal P,\ T\sim D}\Big(M(h)=M(h')\ \wedge\ T\in S_\gamma(h,h')\Big)\\
&\le\ \frac{2\,\bar\delta_{\mathcal P}(\pi)}{c(\gamma)}.
\end{split}
\end{equation*}
\end{proof}

\subsection{Proof of Corollary~\ref{cor:rep_match}}
\begin{proof}
For convenience, write $\ell(h):=\ell_D^\gamma(h)$.

\textbf{1. (ii) implies each $M_j$ is a function of $\ell$.}
Fix $j\in\{1,2\}$. 
Assumption (ii) says $\ell(h)=\ell(h')\Rightarrow M_j(h)=M_j(h')$, so we may define a map $a_j$ on the range of $\ell$ by $a_j(\ell(h)):=M_j(h)$; this is well-defined by (ii). 
Thus,
\begin{equation}\label{eq:rep_match_1}
M_j(h)=a_j(\ell(h)),
\end{equation}
almost surely under the history distribution.

\textbf{2. (i) vanishing pair-regret + (iii) implies $\ell$ is a function of each $M_j$.}
Fix $j$. 
By (i) and Lemma~\ref{lem:alias_bound} applied to $\pi_j$ and $M_j$,
\begin{equation*}
\begin{split}
&\mathbb E_{(h,h')\sim\mathcal P}\Big[\mathbf 1\{M_j(h)=M_j(h')\}\cdot D(S_\gamma(h,h'))\Big]\\
&\le\ \frac{2\,\bar\delta_{\mathcal P}(\pi_j)}{c(\gamma)}\ \to\ 0.
\end{split}
\end{equation*}
Since the integrand is nonnegative, this implies
\begin{equation}\label{eq:rep_match_2}
\mathbf 1\{M_j(h)=M_j(h')\}\cdot D(S_\gamma(h,h'))=0,
\end{equation}
for $\mathcal P$ almost everywhere on $(h,h')$.
Now suppose (for $\mathcal P$ almost everywhere pairs) that $M_j(h)=M_j(h')$. 
Then \eqref{eq:rep_match_2} gives $D(S_\gamma(h,h'))=0$.
By $\gamma$-completeness (iii), $D(S_\gamma(h,h'))=0$ implies $\ell(h)=\ell(h')$ (contrapositive).
Therefore, $M_j(h)=M_j(h')\Rightarrow \ell(h)=\ell(h')$ for $\mathcal P$ almost everywhere on $(h,h')$, which means that $\ell$ is almost surely a function of $M_j$. 
Concretely, we can define $b_j$ on the range of $M_j$ by $b_j(M_j(h)):=\ell(h)$; this is well-defined almost surely because $M_j(h)=M_j(h')$ forces $\ell(h)=\ell(h')$.
Hence
\begin{equation}\label{eq:rep_match_3}
\ell(h)=b_j(M_j(h)),
\end{equation}
almost surely.

\textbf{3. Composition to obtain mutual recodings.}
Using \eqref{eq:rep_match_1} for $j=1$ and \eqref{eq:rep_match_3} for $j=2$,
\begin{equation*}
M_1(h)=a_1(\ell(h))=a_1(b_2(M_2(h)))\ :=\ \varphi(M_2(h)),
\end{equation*}
where $\varphi:=a_1\circ b_2$. Symmetrically,
\begin{equation*}
M_2(h)=a_2(\ell(h))=a_2(b_1(M_1(h)))\ :=\ \psi(M_1(h)),
\end{equation*}
where $\psi:=a_2\circ b_1$.
This proves the claimed mutual recodability, almost surely on the support of $\mathcal P$.
\end{proof}
\end{document}